\documentclass[11pt, a4paper]{article}
\usepackage{physics}
\usepackage{amsfonts, amsmath, amssymb}
\usepackage{amsmath}
\usepackage{amsfonts}
\usepackage{amssymb}
\usepackage{fancyhdr}
\usepackage[ruled, lined, linesnumbered, commentsnumbered, longend]{algorithm2e}
\usepackage{amsthm}
\usepackage{graphicx}
\usepackage{float}
\usepackage{subfig}

\newtheorem{theorem}{Theorem}

\usepackage{amsmath}
\DeclareMathOperator*{\argmin}{argmin}

\begin{document}

\title{Riemannian Geometry Approach for Minimizing Distortion and its Applications}

\author{Dror Ozeri \\
Rafael–Advanced Defense Systems
Israel \\ 
email: {droro@rafael.co.il} }

\date{}
\maketitle

\begin{abstract}
Given an affine transformation $T$, we define its Fisher distortion $Dist_F(T)$. We show that the Fisher distortion has Riemannian metric structure and provide an algorithm for finding mean distorting transformation - namely - for a given set $\{T_{i}\}_{i=1}^N$ of affine transformations, find an affine transformation $T$ that minimize the overall distortion  $\sum_{i=1}^NDist_F^{2}(T^{-1}T_{i}).$ The mean distorting transformation can be useful in some fields -in particular we apply it for rendering affine panoramas.
\end{abstract}

\section{Introduction}
Registration algorithms between different objects are common in computer vision and engineering. Most important usages are planar panoramas \cite{brown2007automatic}, point cloud and point set registration \cite{makovetskii2017affine, du2010affine}. \par
A common practice of the registration process is choosing one of the objects to be the reference object (i.e., the identity transformation). However, unlike rigid transformation, non-rigid transformation can visually distort the original object; choosing an arbitrary or wrong reference object may lead to significant distortions. 

\begin{figure}[H]
\includegraphics[width=\textwidth]{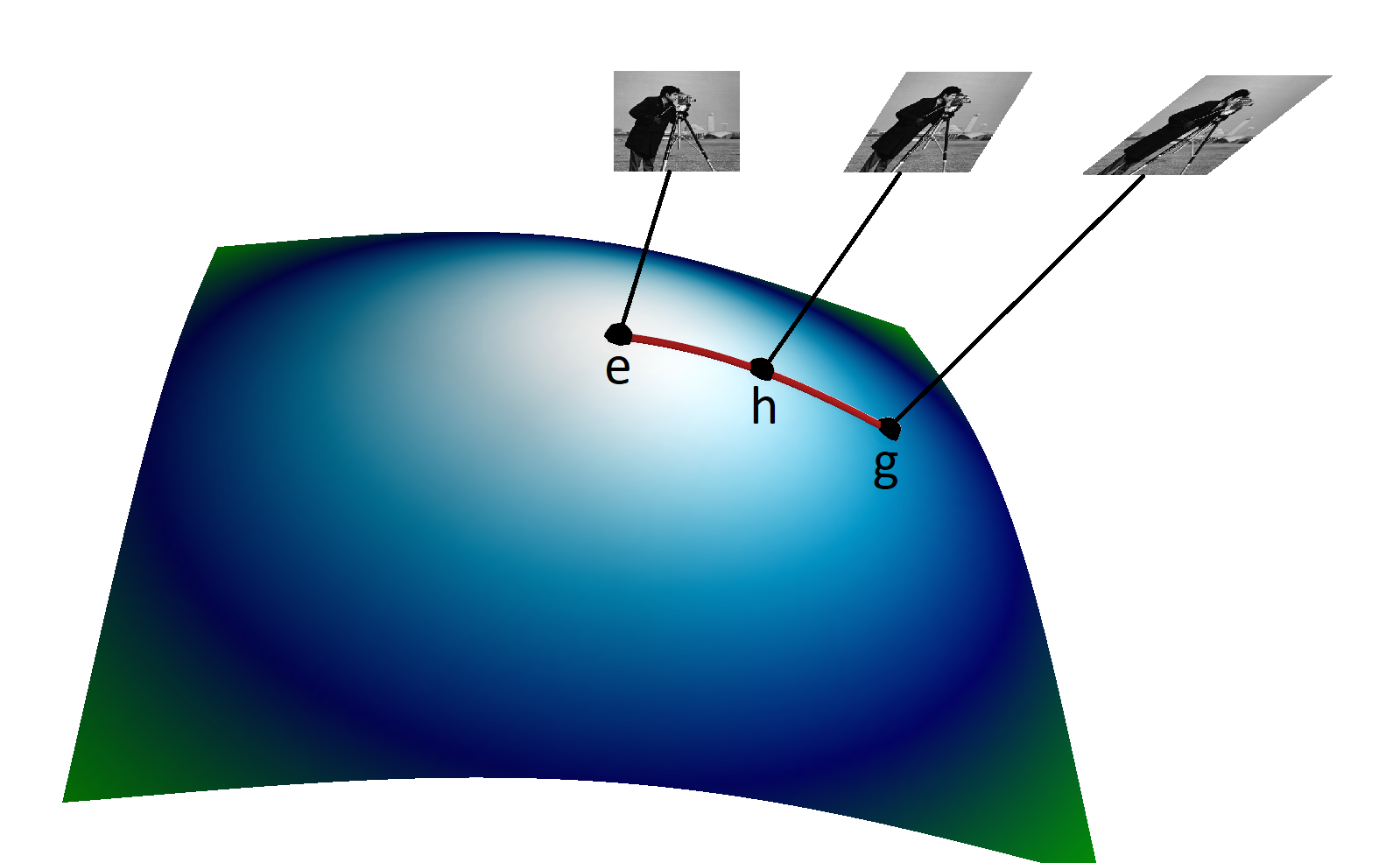}
\caption{ The space of lower triangular matrices with positive diagonal $L_2^+$. The Identity matrix $e$ is the origin and does not distort the image. The matrix  $h$ is located somewhere on the manifold and distorts the image. The matrix  $g$ 
is located further on the manifold and further distorts the image. The distortion is proportional to the distance symbolized by the red curve (see Theorem \ref{main_theorem}).}
\label{fig:main_fig}
\end{figure}

\begin{figure}[H]
\centering
\subfloat[]{\label{fig:a}\includegraphics[height=0.25\textwidth]{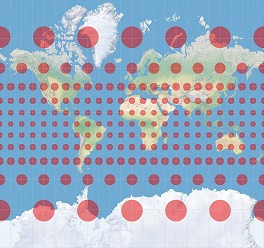}}
\subfloat[]{\label{fig:b}\includegraphics[height=0.25\textwidth]{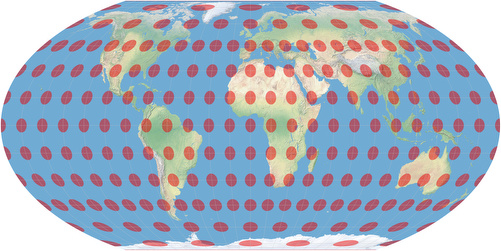}}

\caption{\textit{ Tissot's indicatrix} of the Mercator projection \textbf{(a)} and The Wagner IV projection \textbf{(b)} . The Mercator projection preserves angles but distorts areas, the Wagner IV projection preserves areas but distorts angles. Projection images were taken from \cite{projection}}
\label{fig:tissot}
\end{figure}

\subsection{Related Work}

\subsubsection{Choosing the Best Reference on Planar Panorama}
The problem of estimating the best reference image in a panorama is surprisingly understudied. In \cite{bellavia2015estimating}, Fabio Bellavia et al. investigated this issue with a different distortion measure. In \cite{capel2004image}, Capel suggests taking the central image in the panorama. Indeed in  many implementations of panorama this method is chosen (for example: OpenCV implementation \cite{bradski2008learning}).  \par 
\space
In this study, we show that a natural geometric notion of distortion can be defined and minimized, yielding a new reference plane that is not necessarily one of the initial planes. Moreover, the method is indifferent to the dimension of the space involved and can be applied not only for images. We focus only on the space of affine transformations. \par
As far as we know, reducing distortion using this method is a novel method and was never studied before.

\subsubsection{Map Projection Distortion}
For natural reasons, the concept of distortion has roots in the field of cartography (see \cite{krisztian2019comparing}). First Introduced by Nicolas Auguste Tissot as\textit{ Tissot's indicatrix}  (ellipse of distortion) in 1878 \cite{tissot1878memoire} (see Fig \ref{fig:tissot}). Airy, Kavrayskiy \cite{doi:10.1080/14786446108643179, kavrajskij1934matematiceskaja} and others developed several distortion measures; one of those criteria is the  popular Airy-Kavrayskiy criterion:

\begin{equation}
\label{eqn:airy_Kavrayskiy}
E_{AK} = \sqrt{\int\epsilon_{AK}^2d\sigma} =\sqrt{\int [log^2{\frac{a}{b}}+log^2{ab}]d\sigma} =  \sqrt{\int [log^2{a}+log^2{b}]d\sigma},
\end{equation}
where a $a$ and $b$ are the infinitesimal linear scales of the ellipse. The first term $\epsilon_a^2 = \log^2{\frac{a}{b}}$ measures the angular distortion and the second term  $\epsilon_p^2 = \log^2{ab}$ measures the area distortion (see \cite{krisztian2019comparing}). Laskowski \cite{laskowski1989traditional} has linked between $a$ and $b$ and the singular values of the local Jacobi matrix of the projection, yielding precisely the Fisher distortion that we define later in this paper. In the following sections, we will present an alternative geometric derivation of the Fisher distortion.

\section{Preliminaries and Definitions}
\subsection{Riemannian manifolds and the space of positive definite matrices}
The space of positive definite matrices $P_n^+$, has been the focus of much research in recent years. It has a rich Riemannian structure. Interesting metrics on $P_n^+$  include Euclidean, Inverse Euclidean, Cholesky-Euclidean, log-Cholesky \cite{lin2019riemannian} and the Fisher metric:

\begin{equation}
\label{eqn:fisher}
d_F(A,B) =\sqrt{||log(A^{-\frac{1}{2}}BA^{-\frac{1}{2}})||^2}=\sqrt{\sum_{i=1}^N\log^2(\lambda_{i}(A^{-1}B))},
\end{equation}
where $\lambda_i $ are the matrix eigenvalues \cite{bhatia2003exponential}. The Fisher metric is fundamental, holds a variety of important invariant properties and will be the main focus of this paper.

\subsection{Fréchet mean}
The Fréchet mean on a metric space is a generalization of the arithmetic mean (or center of mass) in Euclidean space. It is well known that for $\{x_i \in \mathbb{R}^n\}_{i=1}^N$ the arithmetic mean of the set $\overline{x}=\frac{\sum_{i=1}^N x_i}{N}$ minimize the expression $\sum_{i=1}^N||x-x_i||^2$. A Fréchet mean of a set of points $\{\nu_i \in M\}_{i=1}^N$ in a metric space $M$ is an element of the metric space that globally minimize the expression $\sum_{i=1}^Nd^2(x,\nu_i)$, where $d$ is the metric. Fréchet means are used in a variety of scientific and mathematical applications.\par 
The space $P_n^+$ endowed with the Fisher metric holds an important property relevant for our case -  Every finite subset of $P_n^+$ has a unique Fréchet mean (see \cite{BINI20131700}).

\subsection{Cholesky Decomposition}
It is known that every positive definite matrix $P$ has a decomposition of $P=LL^t, $ where $L$ is a lower triangular matrix; This is called the Cholesky decomposition of the matrix. This decomposition is unique if we restrict the diagonal elements of $L$ to be positive, namely the Lie group $L_n^{+}$. Furthermore, the map $\Phi:x\to xx^t$ is a diffeomorphism between the group $L_n^{+}$ and the $P_n^{+}$ as Riemannian manifolds. We call the map $\Phi$ - the Cholesky map, and $\Phi^{-1}(P)$ is called the Cholesky factor of $P$.

\subsection{Pullback of a metric}
A common notion in Riemannian geometry is a pullback of a metric. The pullback of a metric is a new metric defined on the preimage of a map. In particular we are interested in the pullback of the Fisher metric into the space $L_n^+$ using the Cholesky map:

\begin{equation}
\Phi^*d_F (l_i,l_j) = d_F(\Phi(l_i), \Phi(l_j)).
\end{equation}

\subsection{QR Decomposition}
Throughout this paper we will use the following version of the QR decomposition:

\begin{theorem}\label{qr_theorem}
Any matrix $A\in GL_n(\mathbb{R})$ can be decomposed \textbf{uniquely} to $A=LQ$ where $L\in L_n^+$ and $Q\in O_n$ and $O_n$ is the set of orthogonal transformation of $\mathbb{R}^n$.
\end{theorem}

\subsection{Definition of the Fisher Distortion}
Let $f(\textbf{x}):\mathbb{R}^n\to\mathbb{R}^n$ be an affine transformation, then $f(\textbf{x})=A\textbf{x}+\textbf{b}$, for some $A\in GL_{n}(\mathbb{R})$ and $\textbf{b}\in \mathbb{R}^n$.
The Fisher distortion of $f$ is

\begin{equation}
Dist_F(f)=Dist_F(A)=\sqrt{\sum_{i=1}^N\log^2(\sigma_{i})},
\end{equation}
where $\sigma_{i}$ are $A$ singular values.

\section{Mean Distorting Transformation Problem}
The mean distorting transformation [MDT] problem is the following:
Let $\{A_i \in GL_{n}(\mathbb{R})\}_{i=1}^N$. Find $\Lambda(A_{1},A_{2},...,A_{n}) \in  GL_{n}(\mathbb{R})$ that satisfies:

\begin{equation}
\Lambda(A_{1},A_{2},...,A_{N}) = \argmin_{A\in GL_n(\mathbb{R})} \sum_{i=1}^NDist_F^{2}(A^{-1}A_i).
\end{equation}.

In section \ref{algorithm_section} we will show that $\Lambda(A_{1},A_{2},...,A_{N})$ is unique up to multiplication by the right with orthogonal transformation.

\subsection{Motivation for the Fisher Distortion}
Besides being studied in the map projection domain, two main properties make the Fisher distortion a natural candidate for the Mean Distorting Transformation Problem. \newline
1. \textbf{Rigid Invariant} - The Fisher Distortion is invariant under left and right multiplication by orthogonal transformations. i.e  
\begin{equation}
Dist_F(q_1 \circ T \circ q_2 )= Dist_F(T),
\forall q_1,q_2 \in O_n.
\end{equation}
This property is simply a consequence of the invariance of the singular values under orthogonal matrix multiplication. \newline
2. \textbf{Geometric Mean Property} - Let $\{D_i \in D_n^+\}_{i=1}^N$ be a set of diagonal matrices with \textbf{positive} entries.  Then it is easy to show that the MDT of this set is simply the geometric mean of the matrices:
\begin{equation}
T = \left(\prod_{i=1}^ND_i\right)^{\frac{1}{N}}.
\end{equation}.

\subsection{Relation between the Fisher metric and Fisher Distortion}

Our main result is the following Theorem:
\begin{theorem} \label{main_theorem}
Let $\Phi^*d_F$ be the pullback metric on $L_n^+$ of the Fisher metric $d_F$ defined on $P_n^+$, then 

\begin{equation}\
\Phi^*d_F(l_i, l_j)=2Dist_F(l_i^{-1}l_j)\; \forall l_i,l_j\in L_n^+.
\end{equation}

\end{theorem}
\begin{proof}
First we show that $\Phi^*d_F$ is left invariant metric: 
\newline 
Let  $l_k \in L_n^+$ , then by definition of $\Phi^*d_F$:

\begin{equation}
\begin{aligned}
\Phi^*d_F(l_kl_i,l_kl_j)=d_F(\Phi(l_kl_i),\Phi(l_kl_j)).
\end{aligned}
\end{equation}
But,
\begin{equation}
\begin{aligned}
d_F(\Phi(l_kl_i),\Phi(l_kl_j))=\sqrt{\sum_{m=1}^N\log^2(\lambda_{m}(\Phi(l_kl_i)^{-1}\Phi(l_kl_j))}.
\end{aligned}
\end{equation}
Observe that
\begin{equation}
\begin{aligned}
\Phi(l_kl_i)^{-1} \Phi(l_kl_j)=(l_k^t)^{-1}\Phi(l_i)^{-1}\Phi(l_j)l_k^t,
\end{aligned}
\end{equation}
the matrices similarity means that both admit the same eigenvalues and by \eqref{eqn:fisher} we have
\begin{equation}
\Phi^*d_F(l_kl_i,l_kl_j)=\Phi^*d_F(l_i,l_j), 
\end{equation}
which means the metric $\Phi^*d_F$ is left invariant.
\newline
Now, 
\begin{equation}
\begin{aligned}
&\Phi^*d_F(l_i,l_j)=\Phi^*d_F(I,l_i^{-1}l_j)= d_F(I,\Phi(l_i^{-1}l_j))=\\ &\sqrt{\sum_{m=1}^N\log^2(\lambda_{m}(\Phi(l_i^{-1}l_j)))}.
\end{aligned}
\end{equation}

By simple proprieties of the singular values, we have: \newline $\sigma_m^2(l_i^{-1}l_j)=\lambda_m(\Phi(l_i^{-1}l_j)$), Finally,
\begin{equation}
\begin{aligned}
& \Phi^*d_F(l_i,l_j)=\sqrt{\sum_{m=1}^N\log^2(\lambda_{m}(\Phi(l_i^{-1}l_j))} = \\ 
& 2\sqrt{\sum_{m=1}^N\log^2(\sigma_{m}(l_i^{-1}l_j))}=2Dist_F(l_i^{-1}l_j).
\end{aligned}
\end{equation}
\end{proof}

A conclusion of Theorem \ref{main_theorem} is that Fisher distortion of a matrix in $L_n^+$ is proportional to its distance from the identity matrix (see Fig \ref{fig:main_fig}).

\subsection{Algorithm for finding MDT}\label{algorithm_section}
Let $\{A_i \in GL_n(\mathbb{R}\}_{i=1}^N$ and let $T\in GL_{n}(\mathbb{R})$. The Fisher distortion is invariant under orthogonal  transformation, therefore we can use QR decomposition of Theorem \ref{qr_theorem} and decompose $A_i= L_{i}Q_{i}$ , $T=LQ$ where $L, L_{i}\in L_{n}^+$ and $Q, Q_{i}\in O_{n}$.\newline
Note that $\sum_{k=1}^N Dist_F^{2}(A^{-1}A_i) = \sum_{k=1}^N  Dist_F^{2}(L^{-1}L_i)$.
From Theorem \ref{main_theorem}, we derive that finding an MDT is equivalent to finding a Fréchet mean on the space $L_n^+$ w.r.t the metric $\Phi^*d_F$. 
For simplicity notice that
\begin{equation}
\begin{aligned}
&Mean_{\Phi^*d_F} (L_1,...,L_n)=\\
&\Phi^{-1}[Mean_{d_F}(\Phi(L_1),..,\Phi(L_n)] = \\ &\Phi^{-1}[Mean_{d_F}(A_1A_1^t,..,A_nA_n^t)].
\end{aligned}
\end{equation}
Therefore calculation of the Fréchet mean can be performed on the space $P_n^+$ - resulting in the following simple algorithm:

\begin{algorithm}
    \caption{Mean Distorting Transformation Algorithm}\label{alg:MDT_algo}
    \SetKwInOut{KwIn}{Input}
    \SetKwInOut{KwOut}{Output}
    \KwIn{A set of linear transformations $ \{ A_i \in GL_n (\mathbb{R}) \}_{i=1}^N$ }
    \KwOut{ $ T \in GL_n(\mathbb{R})$, where $T$ minimize $\sum_{i=1}^N {Dist_F^{2}(T^{-1}A_i)}$.}
	Calculate $ P_i = A_iA_i^t$.\\
	Find the Fréchet mean $ \overline{P}$ of  $\{ P_i \}_{i=1}^N$ w.r.t Fisher metric. \\
	$T=\Phi^{-1}(\overline{P})$.\\
\end{algorithm}

In step 2, note that the Fréchet mean on the space of $P_n^+$ has no explicit formula; however, there are several algorithms for obtaining a numerical solution. For example - gradient based algorithms \cite{BINI20131700} and multiplicative power means (MPM) algorithm \cite{7809176}. \par
While our algorithm supplies a unique solution, notice that any matrix $A$ satisfying $AA^t=\overline{P}$ is also a solution. The latter can only happen if  $A=TQ$, where $Q$ is any orthogonal matrix.

\section{Results and Applications}
\subsection{Affine Panorama}
One of the immediate applications of Algorithm \ref{alg:MDT_algo} is in the affine panorama flow. A complete affine panorama flow is described in \cite{brown2007automatic}. In this flow, images $\{I_i\}_{i=1}^N$ of a particular scene are registered to each other by affine transformations $\{T_i\}_{i=1}^N$, where $T_i$ is affine transformation mapping $I_i$ to a common plane. Classically, the transformations are obtained by solving the system of equations $T_i\Vec{p}_{ik} = T_j\Vec{q}_{jk}$ (in the least squares sense), where $\Vec{p}_{ik}$ and $\Vec{q}_{jk}$ are corresponding points in the $I_i$ and $I_j$ images. \par
When solving this system of equations, we have the degree of freedom of multiplying by global affine transformation $T$. Particular choices of $T$ are taking $T=T_1^{-1}$ or $T=T_c^{-1}$ where $c$ is the index of the the central frame in the panorama.
A more natural and symmetric choice is using Algorithm \ref{alg:MDT_algo} on the linear part of the affine transformations $\{A_{i}\}_{i=1}^N$ where $T_{i}\vec{v}=A_{i}\vec{v}+\vec{b}$. After finding the MDT solution, a rigid global correction can be applied to the images in the following way: \newline
First we apply a global rotation  $q=exp[-\frac{\sum_{i=1}^Nlog(q_i)}{N}]$, where $r_iq_i$ is the QR-factorization of $A_i$. Afterward, we apply the unique global shift transformation $S$ that aligns the panorama with the overall image axis, see Fig \ref{fig:panorama_example} for panorama comparison between panorama generated with a reference fixed image and MDT method.

\subsection{3D Point cloud registration}
Less researched than its rigid counterpart, non-rigid point cloud registration appears in the literature, and various affine point cloud registration algorithms were proposed (see \cite{du2010affine, du2008affine}). Similarly to the 2D case, an MDT on the affine transformations can reduce the overall distortion of both point sets. Note that our method has no dimensionality assumption and that Algorithm \ref{alg:MDT_algo} should work well on any dimension.

\section{Further Work}
The natural extension of this work is to find a mean distorting transformation for non-affine transformation. If $f:\mathbb{R}^n\to \mathbb{R}^n$ is $C^1$, we can use the affine definition of the Fisher distortion of the linearization of $f$ (Jacobi matrix $J_f$):
\begin{equation}
Dist_F(f) = \sqrt{ \int Dist_F^2( J_f) d\mu}
\end{equation}
In this case, the above definition is similar to the Airy-Kavrayskiy criterion distortion measure of map projection. \par
The generalized MDT statement of the problem: \newline
Let $\{f_i:D_i\to D\}_{i=1}^N$ be a set of $C^1$ functions. Find a  $C^1$ invertible function $f:D\to D$ that minimizes the overall distortion:
\begin{equation}
\sum _{i=1}^N Dist_F^2(f^{-1} \circ f_{i}).
\end{equation}

A compelling case to study is the case where $f_i\in PGL(3,\mathbb{R})$ are 2D-homographies. Those studies will enable us to apply an MDT method on \textbf{planar} panoramas, which are a generalization of affine panoramas.

\section{Acknowledgments}
The author would like to thank Adi Shasha for fruitful discussions on this topic.

\begin{figure}[H]
\centering
\begin{tabular}{c }
\includegraphics[width=0.5\textwidth]{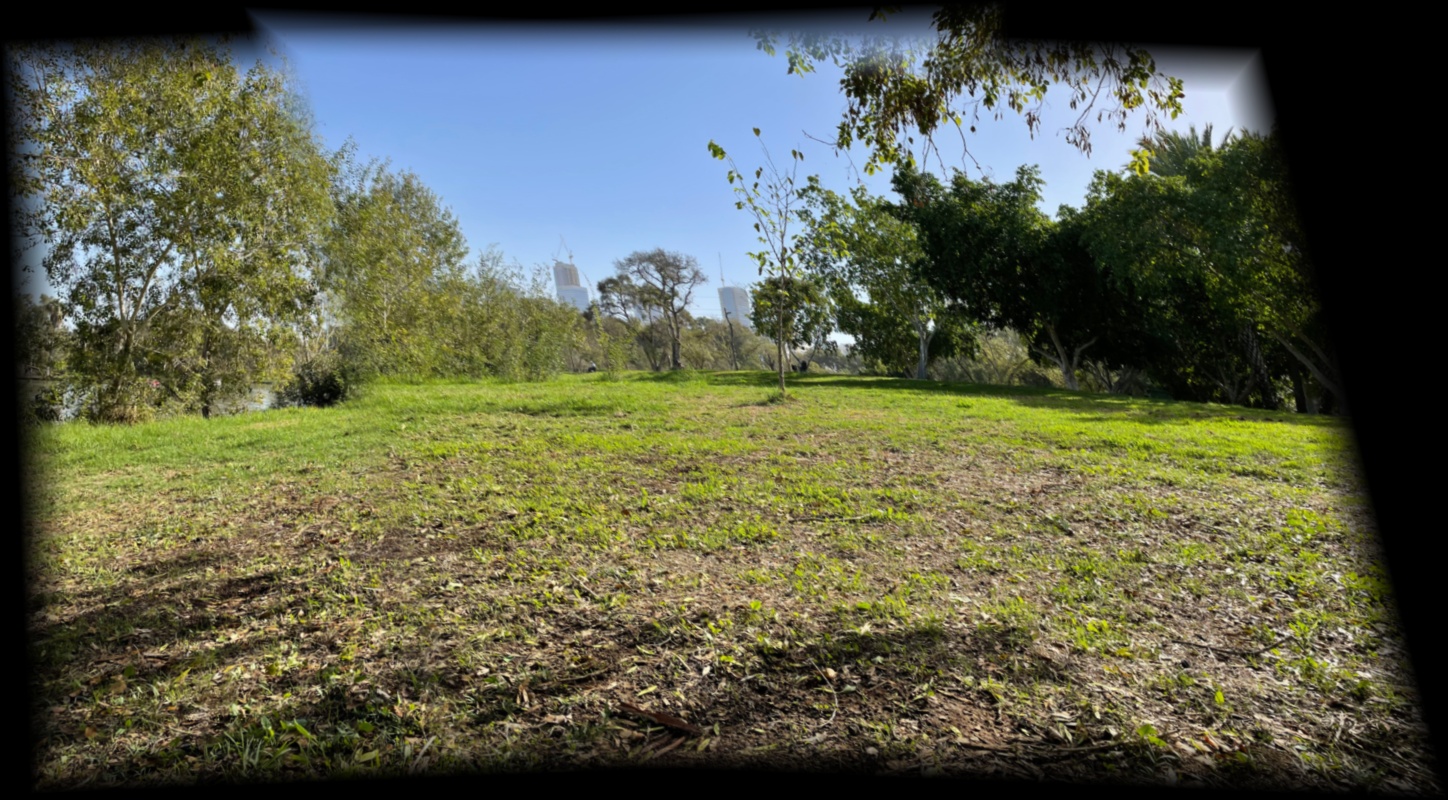} \\
(a)  \\
\includegraphics[width=0.5\textwidth]{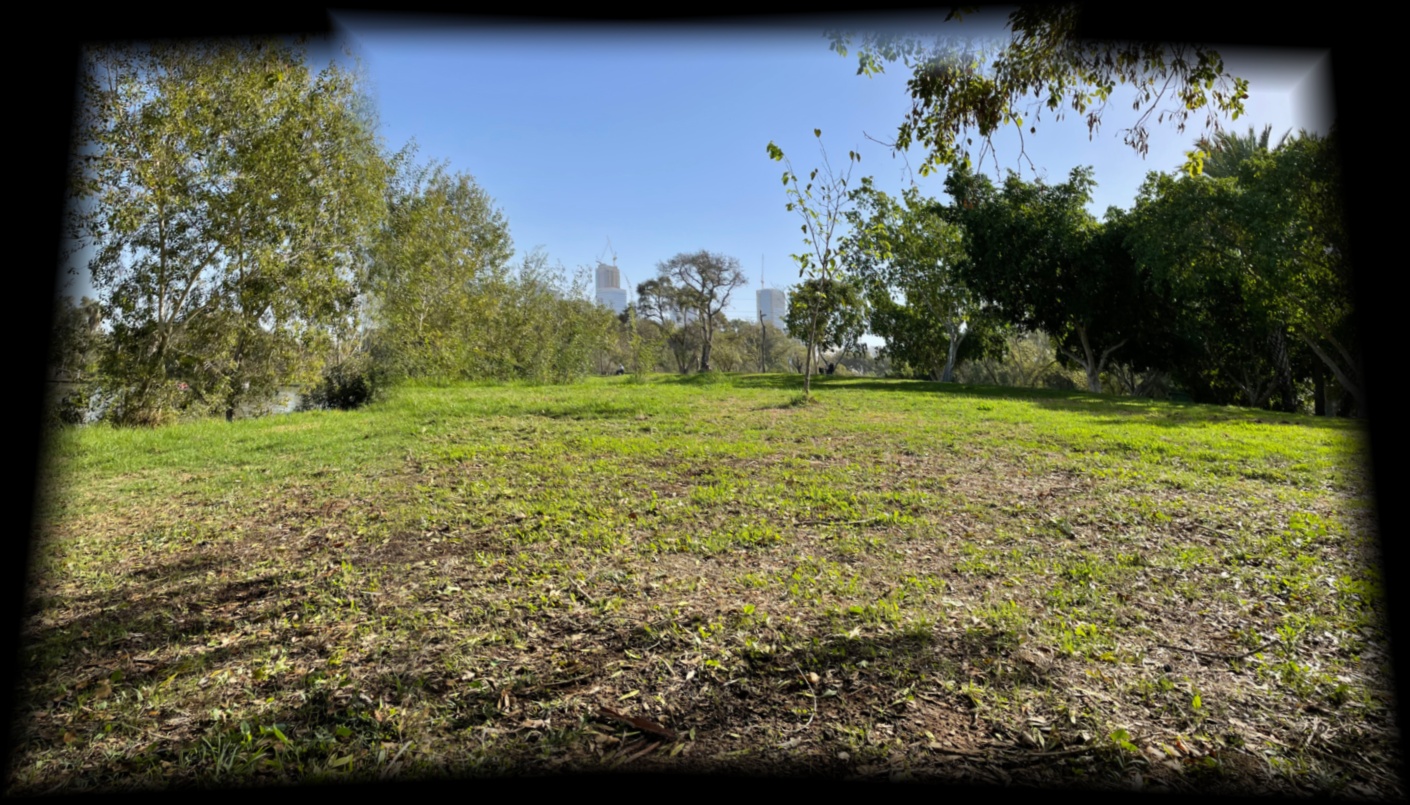} \\
(b) \\
\includegraphics[width=0.8\textwidth]{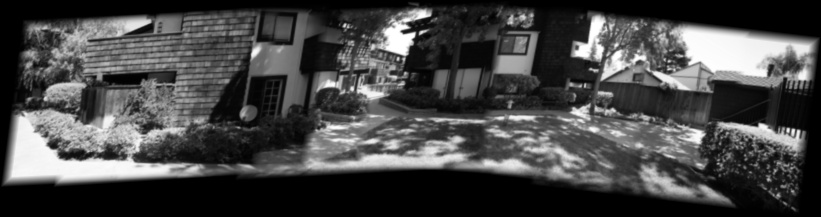}\\
(c) \\
\includegraphics[width=0.8\textwidth]{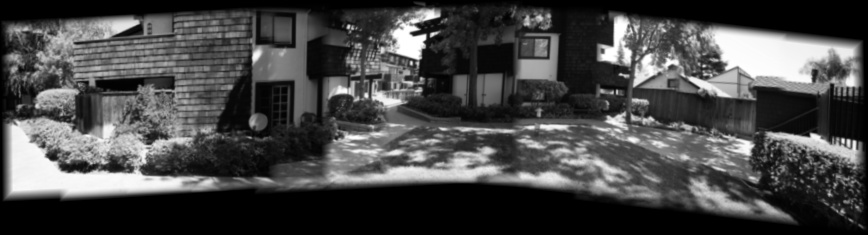}\\
(d)
\end{tabular}
\caption{Visual comparison of panoramas between the fixed reference image method and Mean Distorting Transformation (MDT) correction applied on the fixed reference method: (a) Park panorama (2 images), left image is fixed; (b) MDT correction. (c) Yard Panorama (9 images), rightmost image is fixed. (d) MDT correction.
Raw images of (c) and (d) where taken from Adobe Panorama dataset (see
\cite{brandt2010transform}).
}
\label{fig:panorama_example}
\end{figure}

\bibliographystyle{plain} 

\bibliography{egbib}

\end{document}